\documentclass[10pt]{article}

\usepackage[letterpaper,margin=1in]{geometry}
\usepackage{mathptmx}                 
\usepackage[T1]{fontenc}
\usepackage[utf8]{inputenc}
\usepackage{microtype}
\usepackage{amsmath,amssymb,amsthm}
\usepackage{booktabs}
\usepackage{array}
\usepackage{enumitem}
\usepackage{xcolor}
\usepackage[numbers,sort&compress]{natbib}
\usepackage[colorlinks=true,linkcolor=blue!60!black,citecolor=blue!60!black,urlcolor=blue!60!black]{hyperref}
\usepackage{caption}
\setlist{nosep,leftmargin=1.4em}
\newcolumntype{P}[1]{>{\raggedright\arraybackslash}p{#1}}

\theoremstyle{definition}
\newtheorem{definition}{Definition}
\theoremstyle{plain}
\newtheorem{proposition}{Proposition}

\theoremstyle{remark}
\newtheorem{remark}{Remark}

\newcommand{\Rec}{\mathcal{D}}            
\newcommand{\Rep}{T}                      
\newcommand{\Qs}{\mathcal{Q}}             
\newcommand{\Ent}{H}
\newcommand{\MI}{I}
\newcommand{\Var}{\operatorname{Var}}
\newcommand{\E}{\mathbb{E}}
\newcommand{\Prob}{\mathbb{P}}
\newcommand{\llm}{LLM}
\newcommand{\lqm}{LQM}

\title{\vspace{-1.5em}\textbf{Language Is an Insufficient Substrate for Quantitative Reasoning, and Consequential Domains Need Large Quantitative Models}}

\author{%
  Reuben Vandeventer\\
  Duo Dimensio LLC\\
  \and
  David Imrem\\
  Duo Dimensio LLC\\
  \and
  David J. Wild\thanks{Corresponding author: \texttt{djwild@iu.edu}.}\\
  Luddy School of Informatics, Computing and Engineering, Indiana University\\
  Duo Dimensio LLC
}
\date{Preprint. September 2026.}

\begin{document}
\maketitle

\begin{abstract}
\noindent
The prevailing assumption in applied machine learning is that progress on consequential quantitative decisions such as pricing risk, allocating capital, triaging patients, or containing a network intrusion will follow from progress in large language models (\llm{}s). This paper takes the position that the assumption is mistaken, and that the mistake is structural rather than a matter of present capability. A language model is trained on a representation of the world that was produced by human description; description is a lossy encoding of the quantitative record, and the loss is irreversible: no downstream model, at any scale, can recover from a description what the description did not encode. We formalize this as a property of the \emph{representation} on which a model is trained rather than of the model's capacity, and we identify three further properties that consequential settings demand of a model and that a language substrate cannot supply by construction: reproducibility, lineage from every output back to the source records that produced it, and calibrated uncertainty. We argue that these properties define a distinct model class, which we call the \emph{Large Quantitative Model} (\lqm{}): a foundation-class system trained on a domain's own quantitative record in its native form, holding an explicit and inspectable representation of that domain, with provenance preserved through every transformation and uncertainty quantified by construction. We situate the class against its nearest relatives including tabular and time-series foundation models, world models, and classical single-task quantitative models, and describe one instantiation organized around explicit topological structure, learned dynamics over that structure, and a language model repositioned to take on the role of a human interface. We take the strongest counterarguments seriously, including scaling, fine-tuning, tool use and retrieval, and conclude that each addresses the interface to the quantitative record rather than the substrate. We close with a research agenda for the community.
\end{abstract}

\section{Introduction}
\label{sec:intro}

A language model is a statistical model of how people have written about the world. That is the source of its remarkable breadth, and it is also a precise statement of what it is not: a language model is not a statistical or structural model of the world itself. For many tasks the distinction does not matter, because the task is itself linguistic. For a large and economically central class of tasks it matters decisively, because the object of the decision is not a description but a quantitative record: a loss triangle, a covariance structure, an authentication graph, a dose-response surface, a stream of sensor readings, and so on; and the decision must be defended by reference to that record.

Our position is that for this class of tasks language is an \emph{insufficient substrate}, in a sense we make exact, and that the field should recognize and build a model class native to the quantitative record. We advance two claims, one negative and one constructive.

\paragraph{The negative claim.} Fidelity, reproducibility, lineage and calibration are properties of the \emph{representation} a model is trained on and of the \emph{map} it computes, not of its parameter count or its training budget. When a quantitative record is compressed into natural language, fidelity is lost before any model sees the data, and the loss is irrecoverable in an information-theoretic sense (Section~\ref{sec:fidelity}). When a model produces its output by sampling from a distribution over token sequences, reproducibility is lost by construction (Section~\ref{sec:repro}). When a model's representation of a domain is held implicitly in weights, lineage from output to source is not a query but an inference (Section~\ref{sec:lineage}). And when a model's expression of confidence is itself generated text, calibration must be imposed from outside (Section~\ref{sec:calib}). None of these is repaired by scale, because none is a function of scale.

\paragraph{The constructive claim.} The same four properties, stated positively, define a model class. A \emph{Large Quantitative Model} is (i)~\emph{substrate-native}: trained on the domain's own quantitative record in its native tabular, relational, temporal and high-dimensional form; (ii)~\emph{structurally explicit}: holding its representation of the domain as an inspectable object rather than only in weights; (iii)~\emph{lineage-complete}: with a computable map from every output to the source records that produced it, preserved under every transformation; and (iv)~\emph{uncertainty-calibrated}: emitting predictions with distribution-free coverage guarantees and an explicit abstention state. We argue in Section~\ref{sec:lqm} that this is a genuine class rather than a relabelling of existing quantitative practice, that it is distinct from tabular and time-series foundation models and from world models while sharing their central insight, and that it is buildable: Section~\ref{sec:instantiation} describes one instantiation, and Section~\ref{sec:evidence} gives evidence from a deployed system in which a language model is confined to the interface while a structured quantitative model does the reasoning.

\paragraph{What the position is not.} We do not argue that language models are weak, unhelpful in these domains, or destined for irrelevance. We argue for their correct placement. We agree that the ability of language models to ingest vast amounts of information and process across it enables the model to assume properties that appear quantitative to the human. The architecture we advocate is hybrid: the language model translates human intent into well-formed queries and renders quantitative conclusions as narrative, and a model trained on the domain's own record does the reasoning in between. The language model speaks; the quantitative model knows. Confusing the two roles is, we believe, the dominant reason that enterprise deployments of generative models in quantitative settings have so often failed to produce measurable results.

\paragraph{Why this is a position paper.} The claims here are partly formal, partly empirical and partly a matter of research strategy. The formal core (Section~\ref{sec:substrate}) is elementary and, we think, uncontroversial once stated; what is contested is its implication, which is that the field's current allocation of effort toward making language models safe for quantitative decisions is misdirected relative to building models whose substrate makes them safe by construction. That is a claim about priorities, and we make it as such. Section~\ref{sec:objections} states the strongest counterarguments and our responses. Section~\ref{sec:agenda} sets out what we believe the community should do.

\section{Preliminaries: four properties of a representation}
\label{sec:prelim}

We fix vocabulary that the rest of the paper depends on. Let a \emph{domain record} $\Rec$ be the primitive quantitative data a domain generates: for an insurer, the policies, exposures, claims transactions and treaty terms; for a hospital, the encounters, orders, results and outcomes; for a network defender, the authentication and flow events; for a drug-discovery program, the compounds, assays and activities. Let a \emph{representation} $\Rep(\Rec)$ be any transformation of the record on which a model is trained or over which it reasons. A model is then a map from a representation and a query to an output.

Four properties will concern us. Each is stated as a property of a representation or of a map, so that it can be asked of any model class without prejudice.

\begin{definition}[Fidelity]
\label{def:fidelity}
Let $\Qs$ be a family of queries, each $q\in\Qs$ a function of the record. A representation $\Rep$ is \emph{fidelity-preserving for $\Qs$} if for every $q\in\Qs$, $q(\Rec)$ is a function of $\Rep(\Rec)$; equivalently, in the probabilistic setting, $\Ent\big(q(\Rec)\,\big|\,\Rep(\Rec)\big)=0$.
\end{definition}

Fidelity is relative to a query family, and deliberately so. A representation may be perfectly faithful for one family and useless for another; the summary ``loss development was adverse'' is faithful to the query ``was development adverse?'' and unfaithful to every query about the development triangle itself.

\begin{definition}[Reproducibility]
\label{def:repro}
A system is \emph{reproducible} if its output is a deterministic function of its input, its configuration $\theta$ and its release version: fixing all three, independent executions return identical outputs.
\end{definition}

Reproducibility concerns the map at the time of use, not the process that produced the map. A model may have been trained by a stochastic procedure and still be reproducible in this sense, provided the trained map is frozen, versioned and applied deterministically.

\begin{definition}[Lineage]
\label{def:lineage}
A system is \emph{lineage-complete} if there exists a computable map from every reported output to the set of source records that contributed to it, and that map is preserved under every transformation in the pipeline.
\end{definition}

This is the notion of \emph{where-} and \emph{why-provenance} from the database literature \citep{buneman2001,cheney2009}, carried through model inference rather than stopping at the query layer. We distinguish it sharply from post-hoc attribution: an attribution method estimates which inputs influenced an output; a lineage map \emph{returns} them. The first is an inference about the model; the second is a query against it.

\begin{definition}[Calibration with abstention]
\label{def:calib}
A predictive system is \emph{calibrated} if its stated uncertainty has the frequency interpretation it claims, a nominal $1-\alpha$ prediction set covers the truth with probability at least $1-\alpha$ \citep{vovk2005}, and it \emph{abstains} when the evidence available for a query falls below a stated threshold, rather than reporting a value.
\end{definition}

Two remarks about scope. First, these definitions are indifferent to architecture. Second, they are jointly what supervision of consequential models has come to require. For example, model-risk guidance for financial institutions asks that a model be conceptually sound, that its results be replicable by an independent party, that data flows be documented end to end, and that uncertainty be characterized \citep{sr117}; guidance for insurers' use of AI names data lineage and testing explicitly \citep{naic2023}; the European Union's horizontal AI regulation requires record-keeping, logging and technical documentation sufficient to trace a high-risk system's behavior \citep{euaiact2024}; and regulators of machine-learning medical devices have jointly articulated principles requiring representative data, transparent performance and monitored deployment \citep{gmlp2021}. We return to this in Section~\ref{sec:governance}. The point for now is that the four properties are not a wish list assembled to favour a conclusion; they are, in substance, what the institutions that must answer for consequential decisions already ask of the models that inform them.

\section{Why language is an insufficient substrate}
\label{sec:substrate}

We now argue that a language substrate fails each property by construction. The argument is not that language models are inaccurate on quantitative tasks, although they often are \citep{mirzadeh2025,dziri2023}; empirical inaccuracy can in principle be trained away. The argument is that each failure is located in the representation or in the map, where training cannot reach.

\subsection{Fidelity: description is an irreversible compression}
\label{sec:fidelity}

Consider the record $\Rec$, a description $\Rep(\Rec)$ of it in natural language, and any model output $Y$ produced from the description. Whatever the model, $\Rec \to \Rep(\Rec) \to Y$ is a Markov chain: the model sees the record only through the description. The data-processing inequality \citep[Thm.~2.8.1]{cover2006} then gives $\MI(\Rec;Y)\le\MI(\Rec;\Rep(\Rec))$, and the same holds for any function of the record. The consequence for queries is immediate.

\begin{proposition}[Irrecoverability under lossy description]
\label{prop:irrecoverability}
Let $q$ be a query taking values in a finite set of size $m\ge 2$, and suppose the description is not fidelity-preserving for $q$, so that $\Ent\big(q(\Rec)\,\big|\,\Rep(\Rec)\big)=\eta>0$. Then for every model $Y=g\big(\Rep(\Rec),\xi\big)$, where $\xi$ is any internal randomness independent of $\Rec$, the probability of error satisfies
\[
\Prob\big(Y\neq q(\Rec)\big)\;\ge\;\frac{\eta-1}{\log_2(m-1)} .
\]
In particular, if $\eta>1$ bit, no model built on the description answers $q$ correctly with certainty, and the bound does not depend on $g$.
\end{proposition}
\begin{proof}
Since $q(\Rec)\to\Rep(\Rec)\to Y$ is a Markov chain, $\Ent(q(\Rec)\mid Y)\ge\Ent(q(\Rec)\mid\Rep(\Rec))=\eta$ by the data-processing inequality. Fano's inequality \citep[Thm.~2.10.1]{cover2006} gives $\Ent(q(\Rec)\mid Y)\le h_2(P_e)+P_e\log_2(m-1)$ with $P_e=\Prob(Y\ne q(\Rec))$ and $h_2$ the binary entropy, which is at most one bit. Combining and rearranging yields the bound.
\end{proof}

The proposition is elementary, and that is its force. The right-hand side is a function of the description and the query alone. It does not contain the model's parameter count, its training corpus size, its post-training regime or its inference-time compute. Scaling laws \citep{kaplan2020} describe how well a model approaches the best predictor \emph{of its training distribution}; they say nothing about information that the training distribution never contained. The loss triangle that an analyst summarized into a sentence is not in the sentence, and no model trained on sentences will recover it.

The limitation is older than the technology. \citet{moore1988} argued at length that language stands to the world as a map stands to terrain: it becomes usable only by reducing, grouping and selecting, and it disguises from within how much has been discarded. Their central illustration is quantitative in spirit: in the matching experiments reported by \citet{lehrer1983}, subjects described wines in words and others attempted to match the descriptions back to the samples; as Moore and Carling read the results, performance was often little better than chance, for experienced tasters as well as novices, because the loss occurred at the encoding and no expertise on the receiving end could restore it. \citet{bender2020} make a cognate argument for language models specifically: a system trained on form alone has no access to the meaning that form was used to convey. Our contribution is to observe that in quantitative domains the ``meaning'' at stake is a well-defined mathematical object, and its absence from the substrate can be quantified.

\paragraph{The in-context rejoinder.} A natural objection is that the record can simply be placed in the model's context at inference time, so that the model sees the data rather than descriptions of it. This is correct as far as it goes, and it is why language models are useful for small, self-contained quantitative tasks. It does not rescue the substrate for three reasons. First, the model's \emph{learned} representation of the domain---what it brings to the data---was still formed from descriptions, so the inductive biases it applies to the record in its context are those of a sequence model trained on prose, which differ materially from the biases that tabular and relational data are known to reward \citep{grinsztajn2022}. Second, the record of a consequential domain does not fit in a context window, and when a model runs out of context it interpolates; a mid-sized enterprise network generates authentication graphs of $10^4$--$10^5$ hosts and $10^6$ or more daily edges, and even records that nominally fit are used unevenly across the window \citep{liu2024lost,vallabhaneni2026}. Third, and most fundamentally, in-context provision addresses fidelity alone. The output is still sampled, the reasoning is still not a query against an explicit structure, and the lineage from output to source still runs through a mechanism that cannot be inspected. The remaining three properties fail regardless.

\subsection{Reproducibility: sampling is the mechanism, not a setting}
\label{sec:repro}

A language model's output is a sample from a learned distribution over token sequences. Setting the sampling temperature to zero yields a deterministic map only if the implementation is itself deterministic, which in practice depends on hardware, batching and numerical libraries; more importantly, the intermediate reasoning a model exposes is not a faithful trace of the computation that produced its answer. \citet{turpin2023} show that chain-of-thought explanations can systematically misrepresent the factors that drove a model's prediction, and \citet{jacovi2020} argue that the plausibility of an explanation is no evidence of its faithfulness. For a marketing draft none of this matters. For a reserve estimate, a capital allocation or a containment decision it is disqualifying: the requirement is not merely a correct answer but an answer that is the same answer tomorrow, that decomposes, and that can be reconstructed by an independent party from the same inputs \citep{sr117}.

The standard response has been to build a deterministic shell around the stochastic core: retrieval, guardrails, verifiers, logging of prompts and completions, and human review. This is engineering in good faith, and some of it is valuable. But when a system must be wrapped in an exoskeleton to become safe for its use, the honest engineering conclusion is that the substrate was wrong for the use. The properties the exoskeleton supplies are exactly the properties a quantitative substrate has natively.

\subsection{Lineage: attribution is not provenance}
\label{sec:lineage}

Lineage is the property that most sharply separates the classes, because it is structural rather than behavioural. A model whose representation of a portfolio is a vector of weights can be interrogated after the fact by attribution methods that estimate which inputs mattered. A model whose representation of the portfolio is an explicit structure whose elements are sets of named records can be interrogated by selecting an element. The first is an inference about the model, subject to the well-known instabilities of post-hoc explanation \citep{rudin2019}; the second is a query with a deterministic answer. Where the ``source data'' of a model is the compressed statistical residue of a very large text corpus, there is no lineage to recover: the question ``which records contributed to this number?'' has no mechanical answer because the number was not computed from records.

Retrieval-augmented generation \citep{lewis2020rag} partially restores lineage at the document level, and we regard it as a genuine step in the right direction. But it returns the documents that were retrieved, not the records from which a reported quantity was computed, and the computation between retrieval and output remains the generation mechanism of Section~\ref{sec:repro}.

\subsection{Calibration: verbalized confidence is generated text}
\label{sec:calib}

A language model can be asked how confident it is, and it will answer; the answer is more text. Empirical work finds that verbalized confidence is imperfectly coupled to correctness, with systematic overconfidence under common elicitation methods \citep{xiong2024}, and that even where a model's self-assessment is informative in distribution it degrades under shift \citep{kadavath2022}. This is unsurprising given the substrate: the model has learned how people express confidence, not how often they are right when they do. Calibration can be imposed from outside, by temperature scaling on held-out data \citep{guo2017} or by wrapping the model in a conformal procedure \citep{angelopoulos2023}. But this again supplies from outside a property that a model operating on the quantitative record can carry natively, together with the abstention state of Definition~\ref{def:calib}, which a language model has no natural way to enter: a model trained to continue text does not, by default, decline to continue.

\subsection{Empirical corroboration}

The formal argument predicts a characteristic failure mode: semantically coherent but quantitatively unreliable output, improving in fluency faster than in fidelity across model generations. This is what the empirical record shows. Controlled perturbations of numerical values and irrelevant clauses in grade-school arithmetic problems produce large accuracy drops in frontier models \citep{mirzadeh2025}; compositional tasks with well-defined algorithmic structure show performance collapsing with problem depth in a way consistent with pattern matching over surface form rather than execution of the underlying computation \citep{dziri2023}. In the setting closest to our own argument---tabular prediction on modest datasets---models designed for the statistical character of the data, whether gradient-boosted trees or purpose-built tabular transformers, continue to outperform general-purpose deep models and have motivated a distinct programme of tabular foundation models \citep{grinsztajn2022,hollmann2023,hollmann2025,vanbreugel2024}. At the level of practice, an industry survey of enterprise generative-AI deployments reported that the large majority produced no measurable business impact, with failure to integrate into existing workflows and to retain and use organizational context cited as recurring obstacles \citep{mitnanda2025}; we cite this as an indicator of practitioner experience rather than as a controlled finding, since its methodology has not been peer reviewed. None of this evidence is decisive on its own. Taken together with Section~\ref{sec:fidelity}, it is what one would expect if the substrate, not the model, were the binding constraint.

\section{What consequential settings require}
\label{sec:governance}

The industries and institutions where the economic stakes of quantitative decisions are highest are also those where the \emph{form} of the model is governed, not merely its outcomes. This is the part of the argument that the machine-learning community, whose benchmarks measure outcomes, most often underweights.

Table~\ref{tab:governance} summarizes how four governance regimes engage the four properties of Section~\ref{sec:prelim}. We have deliberately chosen regimes from different sectors: banking, insurance, medical devices and the EU's horizontal AI regulation. The mapping is conservative: we claim the properties are necessary for compliance in each regime, not that possessing them is sufficient.

\begin{table}[t]
\centering\small
\caption{Governance regimes and the representational properties they engage. Primary sources only. Emergency management and public-safety decision support are not governed by a comparable regime but face the same demands informally, through after-action review and liability.}
\label{tab:governance}
\begin{tabular}{P{3.3cm}P{6.6cm}P{5.3cm}}
\toprule
\textbf{Regime} & \textbf{Operative requirement} & \textbf{Property engaged} \\
\midrule
Model-risk guidance for US banks, SR 11-7 / OCC 2011-12 \citep{sr117} & Conceptual soundness; independent validation including outcomes analysis; documentation sufficient for a third party to understand and replicate; ongoing monitoring & Reproducibility (replication); lineage (data-flow documentation); fidelity (soundness of representation); calibration (outcomes analysis) \\
\addlinespace
NAIC Model Bulletin on insurers' use of AI systems \citep{naic2023} & Written AI-systems programme; governance over data including lineage and quality; testing; documentation and third-party oversight & Lineage (named explicitly); reproducibility (testing) \\
\addlinespace
EU Artificial Intelligence Act, high-risk obligations \citep{euaiact2024} & Technical documentation; data governance; record-keeping and logging; accuracy and robustness; human oversight & Lineage (logging and record-keeping); reproducibility; calibration (robustness) \\
\addlinespace
Good Machine Learning Practice for medical devices \citep{gmlp2021} & Representative data; independent test sets; performance transparent to users; monitored deployment & Fidelity (data representativeness); calibration; reproducibility \\
\bottomrule
\end{tabular}
\end{table}

Hold a stochastic language model against these requirements and the mismatch is stark. What is the conceptual soundness of next-token prediction as a theory of credit risk, of claim severity, of lateral movement through a network? How is a model validated whose training corpus is the written internet? How is an independent replication performed of a reasoning chain that does not reproduce run to run? How is an output attributed to source data when there is no lineage?

The constructive inversion is the heart of our position. The properties governance demands such as reproducibility, lineage, calibrated uncertainty, inspectable representation, are natural properties of a model built on the quantitative record and unnatural ones of a model built on language. A model trained on the actual loss data, with every transformation recorded, with uncertainty quantified by construction, does not need an exoskeleton. Governance is not the obstacle to machine learning in consequential domains. Governance is the specification for what the right model class must look like. The field has been trying to make the wrong class compliant instead of building the class that is compliant by construction.

\section{The precedent: domain-native substrates elsewhere in machine learning}
\label{sec:precedent}

Several of the most consequential research programmes of the past decade are, in effect, substrate arguments that have already been accepted.

\paragraph{World models.} The world-model program \citep{ha2018,lecun2022,bruce2024} holds that an agent which must reason about objects, forces and causality should be trained on observations in which objects, forces and causality are present, and should learn a predictive model of that environment's dynamics rather than a model of descriptions of it. The argument is not that language models are weak; it is that language is the wrong substrate for physical reasoning, and that amassing the real data of the domain and training on it yields reasoning that text cannot deliver. We ask the obvious question: why should this logic apply to physics and not to a portfolio, a patient population or an enterprise network? Each has dynamics as real as any physical system---heavy-tailed loss processes, dependency structures that propagate shock, regime changes that invalidate yesterday's calibration---observable only in quantitative data and learnable only from it. A model that has read every regulatory filing ever written has still never observed a covariance matrix evolve through a crisis.

\paragraph{Scientific foundation models.} Protein structure prediction was transformed not by a model trained on the literature about proteins but by one trained on sequences and structures \citep{jumper2021}. The lesson generalizes: where a domain has a native quantitative substrate, models trained on that substrate have repeatedly outperformed models trained on discourse about it.

\paragraph{Tabular and time-series foundation models.} Closer to our concern, a growing literature builds foundation models directly on tabular \citep{hollmann2023,hollmann2025} and time-series data \citep{ansari2024,das2024,woo2024}, motivated by the recognition that the statistical character of such data---non-stationarity, heavy tails, heterogeneous feature types, small samples---is not captured by architectures and objectives inherited from language. \citet{vanbreugel2024} argue, in a position paper we regard as a natural companion to this one, that tabular foundation models should be a research priority precisely because so much consequential data is tabular. We agree, and we go further: substrate-nativeness is necessary but not sufficient. A tabular foundation model can be substrate-native and still fail lineage and structural explicitness, because its representation of the domain lives in weights. Section~\ref{sec:lqm} makes the distinction precise.

\paragraph{Decoupling in agentic systems.} In cybersecurity, a domain where the state is a graph and the actions are consequential, the emerging design pattern is to remove the language model from the decision loop and confine it to the interface \citep{vallabhaneni2026}; Section~\ref{sec:evidence} reports this evidence in detail. The pattern is the same as the world-model migration in miniature: the community built structured models because language-based agents kept failing in environments they could describe perfectly.

\section{The position: Large Quantitative Models}
\label{sec:lqm}

\subsection{Definition}

\begin{definition}[Large Quantitative Model]
\label{def:lqm}
A \emph{Large Quantitative Model} is a foundation-class system for a quantitative domain that is
\begin{enumerate}[label=\textup{(P\arabic*)}]
\item \emph{substrate-native}: trained on and reasoning over the domain's own quantitative record in its native form (tabular, relational, temporal, high-dimensional), without an intervening compression through language;
\item \emph{structurally explicit}: holding its representation of the domain as an explicit, inspectable object---entities, their relations and their geometry---over which it reasons, so that what the model believes the domain \emph{is} can be examined separately from what it predicts;
\item \emph{lineage-complete} in the sense of Definition~\ref{def:lineage}; and
\item \emph{uncertainty-calibrated} in the sense of Definition~\ref{def:calib}, with an explicit abstention state.
\end{enumerate}
It is \emph{foundation-class} in that it is trained across the breadth of the domain's record rather than for a single task, such that individual tasks become conditioned inferences over a structure the model already holds.
\end{definition}

Each property answers a specific failure of the language substrate identified in Section~\ref{sec:substrate}; the definition is the negative argument stated positively. Two of the properties deserve comment because they are the ones most easily lost.

Structural explicitness (P2) is a deliberate departure from the end-to-end orthodoxy of contemporary deep learning, and we make it with our eyes open to the cost. An explicit representation constrains what the model can learn. We hold, with \citet{rudin2019}, that in consequential settings the constraint is worth paying for, because an inspectable representation is what converts ``show your reasoning'' from a research project into a query. We also observe that explicitness does not preclude learning: a representation can be explicit in its structure and learned in its parameters, as a graph whose vertices are sets of records and whose annotations are learned functions of them.

Lineage-completeness (P3) is stronger than it may appear, because it must survive every transformation, including the ones that make a system useful: aggregation, projection, dimension reduction, learned encoding. A pipeline that preserves provenance through ingestion and loses it at the first learned layer is not lineage-complete. The design consequence, which the instantiation in Section~\ref{sec:instantiation} takes seriously, is that the explicit structure of (P2) must be built by operations whose inverse images are computable, and learned components must consume that structure rather than replace it.

\subsection{Neighbouring classes}

Table~\ref{tab:classes} positions the class against its nearest relatives. The comparisons are at the level of the class's defining commitments; individual systems may exceed their class.

\begin{table}[t]
\centering\small
\caption{Model classes against the four properties. ``By construction'' means the property follows from the class's defining commitments; ``by addition'' means it can be supplied by external machinery; ``no'' means the class's commitments preclude it.}
\label{tab:classes}
\begin{tabular}{P{3.6cm}cccc}
\toprule
\textbf{Class} & \textbf{Substrate-native} & \textbf{Structurally explicit} & \textbf{Lineage-complete} & \textbf{Calibrated + abstention} \\
\midrule
Large language model & no & no & no & by addition \\
LLM fine-tuned on serialized tables & partial & no & no & by addition \\
Tabular / time-series foundation model & yes & no & no & by addition \\
World model (physical) & yes & partial & no & by addition \\
Classical single-task quantitative model & yes & yes & often & often \\
Large Quantitative Model & yes & yes & by construction & by construction \\
\bottomrule
\end{tabular}
\end{table}

The row for classical quantitative models is the one that most needs defending, because it invites the response that the class we propose is quantitative practice with new branding. The difference is the difference between a phrase-book and a language model. Classical quantitative models are narrow, hand-specified, single-task artefacts: one model prices one product, one scorecard rates one segment, each built and validated by hand. They are frequently explicit, reproducible and well-provenanced precisely because they are small. A Large Quantitative Model inherits that lineage and extends it to foundation scale: trained across the breadth of a domain's record, holding a structural representation of the whole domain, such that individual tasks are conditioned inferences rather than bespoke builds. What made language models an era rather than a paper was a repeatable production recipe; the corresponding recipe for the quantitative class is what Section~\ref{sec:instantiation} sketches.

The row for tabular and time-series foundation models is the one closest to us in spirit. They share (P1) and are the strongest existing evidence that (P1) matters. They do not, as a class, commit to (P2) or (P3): their representation of the data is in the weights of a transformer, and provenance from a prediction to the training rows that shaped it is an attribution problem. We regard them as an essential component of the class we describe rather than as competitors to it.

\subsection{Composition, not monolith}

A domain has many quantitative decisions---pricing, reserving, fraud, capital, allocation, surveillance in insurance; triage, dosing, discharge and resource allocation in acute care; detection, investigation and containment in network defence. We do not believe the right end state is a single monolithic model. The natural architecture is many foundational models, each trained on its slice of the domain's record and each satisfying (P1)--(P4), composed and orchestrated by language models at the interface. The composable class is the right unit; the monolith is the wrong one.

\section{One instantiation: refine, structure, learn}
\label{sec:instantiation}

Everything in this section is offered as an existence argument, not as part of the definition. Definition~\ref{def:lqm} says what a Large Quantitative Model is; what follows is one way to build one, which the authors have pursued in insurance and in cybersecurity and which one of us has developed in its representational aspects in drug discovery and emergency response \citep{chen2010,kulkarni2016}. There are surely others. We describe the approach at the level of what each stage contributes and why it is necessary for the four properties, not at the level of specific operators, which are reserved for separate treatment.

\paragraph{Why the record is not trainable as found.} The immediate objection to any foundation model for a quantitative domain is practical, and it is the right objection: the record is fragmented across many systems, schemas and definitional regimes; it is private and regulated; its semantics live in the heads of analysts. Text had the enormous advantage of being self-describing and ambient. This is precisely why Large Quantitative Models have not emerged the way language models did: the binding constraint was never modelling talent or compute but the absence of an industrial process for converting a domain's record into a trainable substrate with lineage intact. Any credible instantiation must therefore begin with data engineering treated as a first-class research problem, not a preliminary.

\subsection{Refine: from record to lineage-complete substrate}

The first stage converts heterogeneous source systems into entity-level analytical data with three properties: exhaustive rather than sampled profiling of every field and value; privacy policy enforced at the column level during transformation rather than by post-hoc redaction; and a provenance graph from every output value back to every contributing source field, in the sense of \citet{cheney2009}. The design principle is that lineage is laid down at the first touch of the data and preserved thereafter, because it cannot be reconstructed later. In our experience this stage is where most deployments of learned systems in enterprises fail \citep{sculley2015}, and it is the stage most often treated as beneath the notice of the modelling community. We think that is a mistake of priorities, and we say so in Section~\ref{sec:agenda}.

\subsection{Structure: an explicit representation with computable inverse images}

The second stage builds the explicit representation of (P2) in a way that delivers (P3) as a by-product. Treat the domain's entities at time $t$ as a finite sample $X_t\subset\mathbb{R}^p$ with a metric appropriate to the feature mix. The construction we have found most serviceable is a development of the Mapper method of \citet{singh2007} within the topological-data-analysis programme of \citet{carlsson2009}: project through a low-dimensional lens, cover the lens image with overlapping regions at a chosen resolution and overlap, reduce the pullback of each region locally, and glue the local pieces into a global complex along their shared entities. Three features of this construction matter for the position, and none of them depends on the particular operators used.

First, the representation is built from the entities' \emph{descriptors} only. Response variables---loss, outcome, default, compromise---are then annotated onto the structure as \emph{contrast} afterwards, with estimators corrected for the overlap between neighbouring regions and gated by an effective-sample-size criterion below which the system abstains. Keeping structure and outcome separate is what makes the subsequent observation that outcomes are organized over the structure a testable claim rather than a circularity. Second, every element of the structure has a computable inverse image: a vertex \emph{is} a set of named records. Lineage in the sense of Definition~\ref{def:lineage} is therefore not an added feature but the inverse-image map of the construction itself. Third, the construction is multiscale and its parameters are reportable, so the same record can be presented coarsely for a portfolio-level view and finely for an entity-level one without re-modelling; stability of the resulting structure under perturbation of the input has been studied for the one-dimensional case \citep{carriere2018} and for multiscale covers \citep{dey2016}, and we are explicit that stability guarantees for general lens dimensions remain partial (Section~\ref{sec:agenda}).

We stress what we do not claim. We do not claim that the computed complex recovers the homotopy type of any underlying continuous object; the entities are a finite sample and the guarantees available from nerve constructions apply to covers of spaces, not of samples. Our use of topology is as an organizing device for local-to-global summarization with computable provenance, and the claims we make about it are the claims that survive that reading.

\subsection{Learn: dynamics over the structure, with the language model at the interface}

The third stage learns over the time-indexed family $(K_t,\varphi_t)_{t\ge1}$ of structures and their contrast annotations. The organizing principle is \emph{structure first, then outcome}: predict how the structure evolves, then predict outcomes conditioned on the structure, rather than predicting outcomes directly from pooled features. The justification we can offer for this factorization is modest and we state it as such.

\begin{proposition}[Conditioning on observed structure does not increase expected uncertainty]
\label{prop:structure}
For any response $Y$, observed structure $K$ and feature information $F$ with finite second moments,
\[
\Ent(Y\mid K,F)\;\le\;\Ent(Y\mid F)
\qquad\text{and}\qquad
\E\big[\Var(Y\mid K,F)\big]\;\le\;\E\big[\Var(Y\mid F)\big].
\]
\end{proposition}
\begin{proof}
The entropy inequality is the statement that conditioning cannot increase conditional entropy \citep[Thm.~2.6.5]{cover2006}. The variance inequality is the law of total variance applied within each level of $F$: $\Var(Y\mid F)=\E[\Var(Y\mid K,F)\mid F]+\Var(\E[Y\mid K,F]\mid F)\ge\E[\Var(Y\mid K,F)\mid F]$; take expectations over $F$.
\end{proof}

\begin{remark}
Proposition~\ref{prop:structure} is a statement about \emph{observed} structure. When the structure at a future horizon is itself predicted, the inequality bounds what the factorization could deliver if the prediction were perfect; it does not guarantee the factorization helps. Whether and when it does is an empirical and theoretical question that we list as open in Section~\ref{sec:agenda}. We regard stating this limit as more useful to the community than overclaiming.
\end{remark}

Calibration (P4) is supplied at this stage by mechanisms that are public and well understood: probability calibration on held-out data \citep{guo2017}, and distribution-free prediction intervals via conformal methods \citep{vovk2005,angelopoulos2023}, chosen because they make no distributional assumption---important under heavy tails---and because under distribution shift they degrade honestly, widening rather than silently losing coverage, with adaptive variants available for streaming operation \citep{tibshirani2019,gibbs2021}. Where a forecast is conditioned on predicted rather than observed structure, the two are calibrated separately and the gap is reported. Together with the abstention rule inherited from the structure stage, the system has three epistemic states rather than one: a value with an interval, a value flagged as poorly calibrated under current shift, or a refusal to report. The third state is the one most often missing from analytical software, and its absence is the source of much of the distrust that quantitative systems attract from the experts who must act on them.

Finally, the language model. In this instantiation it parses intent, selects the appropriate view of the structure, and narrates the computed result---with the narration constrained to the computed object and every quantitative claim in it carrying the identifiers of the records it summarizes. A narrated number whose provenance is a computation over the structure is auditable; a number produced \emph{by} narration is not. That is the entire division of labour, and it is the one Section~\ref{sec:evidence} shows working.

\section{Evidence: decoupling the language model from the decision}
\label{sec:evidence}

The position predicts that in a consequential domain with a structured quantitative state, confining the language model to the interface and placing the decision in a model of the structure will improve both reliability and auditability. We have one deployed test of this prediction that we can report in full.

\citet{vallabhaneni2026} describe an agentic security-operations architecture for detecting and containing lateral movement, evaluated on the Los Alamos National Laboratory multi-source cyber-security events dataset, which spans 58 days of production traffic with verified red-team activity. The state of the problem is an authentication graph of roughly $1.8\times10^4$ hosts and $2.4\times10^7$ daily edges---an object that no current context window can hold, and one whose relational structure is exactly what a serialized prompt discards. The architecture decouples topological from semantic reasoning: a heterogeneous graph attention encoder summarizes the live two-hop neighbourhood of a flagged host into a fixed-dimensional state; a reinforcement-learning policy maps that state to one of five constrained investigative actions; and the language model is restricted to consuming the policy's recommendation and producing an analyst-readable narrative, gated by a critic that verifies cited evidence and by a human-approval boundary for any irreversible action.

Three findings bear directly on the position. First, on held-out red-team events the structured policy achieved precision $0.91\pm0.02$ and recall $0.87\pm0.03$, the highest precision among the graph-based lateral-movement detectors compared, at recall within the range of the strongest of them. Second, the decision itself---subgraph extraction, encoding and policy inference---executed in under 100~milliseconds, while the language model's narrative synthesis accounted for 60\% of the 6.3-second median end-to-end latency: the language model was the slowest and least essential component of the loop, and the one whose output was gated rather than trusted. Third, the auditability that governance requires came from the structured side: an immutable ledger of network state, action, policy version and cited evidence for every decision, action masking that forbids a verdict without corroborating evidence from at least two sources, and tiered autonomy in which read-only reconnaissance is automatic, reversible containment requires the critic's validation, and irreversible intervention requires a human. None of these properties was supplied by the language model; all were supplied by the explicit structure and the constrained policy over it.

We do not present this as a demonstration of a Large Quantitative Model in the full sense of Definition~\ref{def:lqm}: the system is single-task, and lineage from the policy's decision to source authentication events runs through a learned encoder rather than a computable inverse image. We present it as evidence for the division of labour the position rests on, in a domain where the alternative---a language model as drop-in analyst---has been widely proposed, and as an illustration of how much of what governance asks for arrives once the decision is placed on a structured quantitative substrate.

The same pattern is visible, at lower resolution, in the other domains one of us has worked in. In chemogenomics, integrating heterogeneous chemical, biological and clinical evidence into a single explicit structure with inspectable provenance is what made inference over sparse activity data defensible enough to carry through to experimentally confirmed candidates \citep{chen2010,kulkarni2016}. In emergency management, the operative constraint is that a commander must be able to see why a system recommends what it recommends, under time pressure, from data that is incomplete; a structured, provenance-preserving representation is what puts that ``why'' within reach. We claim no more for these than that they are consistent with the position.

\section{Alternative views}
\label{sec:objections}

\paragraph{``Scaling will fix it.''} The most common rebuttal is that language-model limitations in quantitative domains are temporary. The answer is Proposition~\ref{prop:irrecoverability}: the error bound is a function of the description and the query, not of the model. Scale recovers more of what is in the training distribution; it cannot recover what was never encoded. The empirical record agrees, in that each model generation has improved fluency about quantitative domains faster than fidelity to them, which is what one predicts when the substrate is description. And notably, the researchers with the most privileged view of scaling are among those redirecting effort toward non-language substrates for physical reasoning \citep{lecun2022}; they are not betting against scale, but against scale on the wrong data.

\paragraph{``Fine-tune the language model on the domain's data.''} Fine-tuning adjusts a language model's behaviour; it does not change its substrate. A transformer pre-trained on text and fine-tuned on serialized tables is still a next-token predictor whose inductive biases were formed on prose, and the deeper problems survive untouched: outputs remain sampled, lineage remains absent, and the representation of the domain remains in weights no validator can inspect. The tabular foundation-model literature reaches the same conclusion from the empirical side---architectures and objectives must be designed for the statistical character of the data rather than inherited from language \citep{grinsztajn2022,vanbreugel2024}. Fine-tuning is a patch on the interface; the position concerns the foundation.

\paragraph{``Give the language model tools.''} This is the strongest contemporary objection, and we want to state it fairly. Tool use \citep{schick2023}, program-aided reasoning \citep{gao2023pal} and retrieval \citep{lewis2020rag} let a language model delegate computation and look up facts, so that the model need not itself be the substrate for quantitative reasoning. We agree, and we observe that this is our position: the moment the arithmetic is delegated to a calculator and the data to a database, the reasoning core is the tool, and the language model is the interface. What remains contested is the design of the tool. If the ``tool'' is a query against an explicit, provenance-preserving representation of the domain with calibrated outputs, it is a Large Quantitative Model by another name and we have no disagreement. If the tool is a code interpreter executing code the language model wrote, then the reproducibility and lineage of the pipeline are bounded by the reproducibility and lineage of the generation step, which is to say they are not guaranteed. Tool use moves the problem to exactly the place we say it belongs; it does not solve it there.

\paragraph{``Quantitative models already exist; this is quant with new branding.''} Section~\ref{sec:lqm} answers this at length. The short form: classical quantitative models are single-task and hand-built; the class we propose is foundation-scale, holds a representation of the whole domain, and treats tasks as conditioned inferences. The difference is not one of degree.

\paragraph{``Explicit structure sacrifices accuracy.''} This is the objection we take most seriously on scientific grounds, because it may sometimes be true. End-to-end learning over unconstrained representations has repeatedly outperformed structured alternatives when the objective is a single scalar metric on a benchmark. Two responses. First, in consequential settings the objective is not a single scalar; it is accuracy jointly with reproducibility, lineage and calibration, and Section~\ref{sec:governance} argues the latter three are hard constraints. \citet{rudin2019} has made the case that the supposed accuracy--interpretability trade-off is often smaller than assumed when the interpretable model is designed rather than defaulted to; we find the same in practice. Second, structural explicitness in our sense constrains the \emph{representation}, not the learning over it: the annotations and dynamics over an explicit structure can be as expressive as the problem demands. We nonetheless list the characterization of this trade-off as an open problem, because a position paper should be candid about where its position may cost something.

\paragraph{``The data problem is disqualifying.''} Text could be scraped; an insurer's claims history, a hospital's encounters and a network's authentication logs cannot and should not be. This is the objection that is fatal for most would-be builders. But it asserts not that the class is impossible, only that it is impossible without first solving autonomous refinement, privacy enforcement and lineage preservation. That is a statement about sequencing, and it explains why data engineering is the first stage of Section~\ref{sec:instantiation} rather than an afterthought. The fragmentation of quantitative data is not a refutation of the position. It is the reason the class has not emerged spontaneously, and the reason it will be built by those who treat data engineering as seriously as model architecture.

\section{A research agenda}
\label{sec:agenda}

If the position is right, the community's effort is misallocated, and the correction is not to abandon language models but to build the class that sits beneath them. We propose the following, in rough order of foundational priority.

\begin{enumerate}
\item \textbf{A theory of representational adequacy for quantitative domains.} Definition~\ref{def:fidelity} makes fidelity relative to a query family. The natural theoretical programme is to characterize, for a given domain, the minimal representations that are fidelity-preserving for the decision-relevant queries, and to quantify the information loss of standard summarizations---including the linguistic ones that populate every training corpus. Proposition~\ref{prop:irrecoverability} is the first, trivial step.

\item \textbf{Lineage as a first-class learning constraint.} Provenance is well developed for databases \citep{buneman2001,cheney2009} and largely absent from learned systems. We need architectures in which learned components consume explicit structures with computable inverse images and preserve them, and formal notions of \emph{lineage-preserving} layers analogous to equivariance constraints. Where lineage must pass through a learned encoder, as in Section~\ref{sec:evidence}, we need principled bounds on what is lost.

\item \textbf{Structure-first learning under predicted structure.} Proposition~\ref{prop:structure} covers observed structure. The important case is predicted structure at a horizon, where the structure itself carries model uncertainty. When does the factorization $p(K_{t+h}\mid F)\,p(Y\mid K_{t+h},F)$ improve on direct prediction, and by how much? This is a question about model-based prediction in a discrete, combinatorial state space with births and deaths of structural elements, and it is largely open.

\item \textbf{Stability of explicit topological representations.} Stability of Mapper-type constructions is understood in one dimension \citep{carriere2018} and for multiscale covers \citep{dey2016}; general results for the lens dimensions and heterogeneous local reducers used in practice are not available, and their absence is the weakest point in the instantiation of Section~\ref{sec:instantiation}. Governance regimes that require stability analysis of internal models make this more than a theoretical nicety.

\item \textbf{Calibration and abstention under drift, jointly.} Adaptive conformal methods \citep{gibbs2021} address coverage under shift; abstention rules address thin evidence. Consequential deployment needs both, with guarantees that compose, and with the calibration gap between observed-structure and predicted-structure conditioning characterized rather than merely reported.

\item \textbf{Benchmarks that score the four properties.} Every benchmark that drives the field scores outcomes. None we know of scores reproducibility, lineage completeness or calibrated abstention as first-class quantities. Until they do, systems that satisfy them by construction will look no better on leaderboards than systems that do not, and the incentive to build them will remain commercial and regulatory rather than scientific.

\item \textbf{Interface protocols between language and quantitative models.} The division of labour we advocate needs a specification: how intent is translated into a well-formed query against an explicit structure; how a narration is constrained to the computed object; how every quantitative claim in generated text carries the identifiers that make it auditable. This is a research problem in its own right and the one where the language-model community has most to contribute.

\item \textbf{Autonomous refinement as a research topic.} The conversion of fragmented, private, semantically undocumented records into lineage-complete training substrate is treated as engineering. It is the binding constraint on the entire class, and it deserves the attention the field gives to architecture.
\end{enumerate}

\section{Conclusion}

Every major class of machine intelligence has been defined by its substrate. Language models, trained on text, mastered the world of discourse. World models, trained on physical observation, are learning the world of space and matter. The quantitative core of consequential decision-making---where most of the economy's decisions and many of its most serious ones are made, under institutions that will accept nothing less than models that can show their work from the source record forward---has been waiting for its model class. We have argued that language cannot be that substrate, for reasons that are structural and elementary; that the properties governance demands are natural to a model built on the quantitative record and unnatural to one built on language; that the corresponding class can be defined by four properties and distinguished from its neighbours; and that it is buildable, with a first instantiation in hand and evidence that its central architectural commitment---the language model at the interface, a structured quantitative model at the core---delivers what the position predicts.

A language model is a statistical model of how people describe the world, not of the world itself. Models are what they eat. For the quantitative core of consequential decisions, the right diet was never language. It was the record all along, waiting for a model class capable of refining it, giving it structure and learning from it.

\section*{Acknowledgements}
The authors thank colleagues at Indiana University and at Duo Dimensio for discussions that shaped this argument. \emph{[To be completed: funding acknowledgements; the Indiana University partnership; any collaborator credits.]}

\section*{Competing interests}
R.V., D.I. and D.J.W. are affiliated with Duo Dimensio LLC, which develops systems of the kind described in Section~\ref{sec:instantiation}. D.J.W. is additionally a faculty member at Indiana University. The position advanced here is the authors' own.

\small
\bibliographystyle{plainnat}
\bibliography{references}

\end{document}